\documentclass[letterpaper]{article}

\usepackage{times}
\usepackage{graphicx} 
\usepackage{subfigure} 

\usepackage{natbib}

\usepackage{algorithm}
\usepackage{algorithmic}

\usepackage[accepted]{icml2016} 

\icmltitlerunning{Distributed Clustering of Linear Bandits in Peer to Peer Networks}

\usepackage{macros}

\begin{document}

\twocolumn[
\icmltitle{Distributed Clustering of Linear Bandits in Peer to Peer Networks}

\icmlauthor{Nathan Korda}{nathan@robots.ox.ac.uk}
\icmladdress{MLRG, University of Oxford}
\icmlauthor{Bal{\'a}zs Sz{\"o}r{\'e}nyi}{szorenyi.balazs@gmail.com}
\icmladdress{EE, Technion \& MTA-SZTE Research Group on Artificial Intelligence}
\icmlauthor{Shuai Li}{shuaili.sli@gmail.com}
\icmladdress{DiSTA, University of Insubria}

\icmlkeywords{bandits, contextual, distributed, gossip, clustering, regret}

\vskip 0.3in
]

\begin{abstract}
We provide two distributed confidence ball algorithms for solving linear bandit problems in peer to peer networks with limited communication capabilities. For the first, we assume that all the peers are solving the same linear bandit problem, and prove that our algorithm achieves the optimal asymptotic regret rate of any centralised algorithm that can instantly communicate information between the peers. For the second, we assume that there are clusters of peers solving the same bandit problem within each cluster, and we prove that our algorithm discovers these clusters, while achieving the optimal asymptotic regret rate within each one. Through experiments on several real-world datasets, we demonstrate the performance of proposed algorithms compared to the state-of-the-art.
\end{abstract}

\allowdisplaybreaks
\section{Introduction}
Bandits are a class of classic optimisation problems that are fundamental to several important application areas.
The most prominent of these is recommendation systems, and they can also arise more generally in networks (see, e.g., \cite{shuai13NetworkRecommendation,shuai15NetworkRecommendation}).

We consider settings where a network of agents are trying to solve collaborative linear bandit problems.
Sharing experience can improve the performance of both the whole network and each agent simultaneously, while also increasing robustness.
However, we want to avoid putting too much strain on communication channels.
Communicating every piece of information would just overload these channels.
The solution we propose is a gossip-based information sharing protocol which allows information to diffuse across the network at a small cost, while also providing robustness.

Such a set-up would benefit, for example, a small start-up that provides some recommendation system service but has limited resources.
Using an architecture that enables the agents (the client's devices) to exchange data between each other directly and to do all the corresponding computations themselves could significantly decrease the infrastructural costs for the company.
At the same time, without a central server, communicating all information instantly between agents would demand a lot of bandwidth.

\textbf{Multi-Agent Linear Bandits}
In the simplest setting we consider, all the agents are trying to solve the same underlying linear bandit problem.
In particular, we have a set of \emph{nodes} $V$, indexed by $i$, and representing a finite set of \emph{agents}. At each time, $t$:
\begin{itemize}\setlength\itemsep{-0pt}\setlength\leftmargin{0.5cm}
	\item a set of \emph{actions} (equivalently, the \emph{contexts}) arrives for each agent $i$, $\cD_t^i \subset \cD$ and we assume the set $\cD$ is a subset of the unit ball in $\bR^d$;
	\item each agent, $i$, chooses an action (context) $x_t^i \in \cD_t^i$, and receives a \emph{reward} 
$$
r_t^i = (x_t^i) \tr \theta + \xi_t^i,
$$
where $\theta$ is some unknown \emph{coefficient vector}, and $\xi_t^i$ is some zero mean, $R$-subGaussian noise;
	\item last, the agents can share information according to some protocol across a communication channel.
\end{itemize}
We define the \emph{instantaneous regret} at each node $i$, and, respectively, the \emph{cumulative regret} over the whole network to be:
\vspace{-0.15cm}  
$$\rho_t^i := \left(x_t^{i,\ast}\right)\tr \theta -  \bE r_t^i
\text{,\;\; and\;\; }\cR_t:=\sum_{k = 1}^t\sum_{i = 1}^{|V|}\rho_t^i,$$
where $x_t^{i,\ast}:=\argmax_{x\in\cD_t^i}x\tr \theta$. The aim of the agents is to minimise the rate of increase of cumulative regret.
We also wish them to use a sharing protocol that does not impose much strain on the information-sharing communication channel.

\textbf{Gossip protocol}
In a gossip protocol (see, e.g., \cite{KeDoGe03,avgconsensus,JMB05,JVGKvS07}), in each round, an overlay protocol assigns to every agent another agent, with which it can share information.
After sharing, the agents aggregate the information and, based on that, they make their corresponding decisions in the next round.
In many areas of distributed learning and computation gossip protocols have offered a good compromise between low-communication costs and algorithm performance.
Using such a protocol in the multi-agent bandit setting, one faces two major challenges.

First, information sharing is not perfect, since each agent acquires information from only one other (randomly chosen) agent per round.
This introduces a bias through the unavoidable doubling of data points.
The solution is to mitigate this by using a delay (typically of $O(\log t)$) on the time at which information gathered is used. After this delay, the information is sufficiently mixed among the agents, and the bias vanishes.

Second, in order to realize this delay, 
it is necessary to store information in a buffer and only use it to make decisions 
after the delay has been passed.
In \cite{szorenyi2013gossip} this was achieved by introducing an epoch structure into their algorithm, and emptying the buffers at the end of each epoch.

\textbf{The Distributed Confidence Ball Algorithm (DCB)}
We use a gossip-based information sharing protocol to produce a distributed variant of the generic Confidence Ball (CB) algorithm, \cite{abbasi2011improved,dani2008stochastic,li2010contextual}.
Our approach is similar to \cite{szorenyi2013gossip} where the authors produced a distributed $\epsilon$-greedy algorithm for the simpler multi-armed bandit problem.
However their results do not generalise easily, and thus significant new analysis is needed.
One reason is that 
the linear setting introduces serious complications in the analysis of the delay effect mentioned in the previous paragraphs.
Additionally, their algorithm is epoch-based, whereas we are using a more natural and simpler algorithmic structure.
The downside is that the size of the buffers of our algorithm grow with time.
However, our analyses easily transfer to the epoch approach too.
As the rate of growth is logarithmic, our algorithm is still efficient over a very long time-scale.

The simplifying assumption so far is that all agents are solving the same underlying bandit problem, i.e. finding the same unknown $\theta$-vector.
This, however, is often unrealistic, and so we relax it in our next setup.
While it may have uses in special cases, DCB and its analysis can be considered as a base for providing an algorithm in this more realistic setup, where some variation in $\theta$ is allowed across the network.

\textbf{Clustered Linear Bandits}
Proposed in \cite{gentile2014online,shuai16gclub,shuai16cofiba}, this has recently proved to be a very successful model for recommendation problems with massive numbers of users.
It comprises a multi-agent linear bandit model agents' $\theta$-vectors are allowed to vary across a clustering.
This clustering presents an additional challenge to find the groups of agents sharing the same underlying bandit problem before information sharing can accelerate the learning process.
Formally, let $\{U^k\}_{k = 1,\dots,M}$ be a clustering of $V$, assume some coefficient vector $\theta^k$ for each $k$, and let for agent $i\in U^k$ the reward of action $x_t^i$ be given by
\begin{align*}
r_t^i = (x_t^i)\tr \theta^k + \xi_t^i.
\end{align*}
Both clusters and coefficient vectors are assumed to be initially unknown, and so need to be learnt on the fly.

\textbf{The Distributed Clustering Confidence Ball Algorithm (DCCB)}
The paper \cite{gentile2014online} proposes the initial centralised approach to the problem of clustering linear bandits.
Their approach is to begin with a single cluster, and then incrementally prune edges when the available information suggests that two agents belong to different clusters.
We show how to use a gossip-based protocol to give a distributed variant of this algorithm, which we call DCCB.

\textbf{Our main contributions}
In Theorems \ref{thm:Main} and \ref{thm:DCCB_main} we show our algorithms DCB and DCCB achieve, in the multi-agent and clustered setting, respectively, near-optimal improvements in the regret rates. 
In particular, they are of order almost $\sqrt{|V|}$ better than applying CB without information sharing, while still keeping communication cost low. 
And our findings are demonstrated by experiments on real-world benchmark data.

\section{Linear Bandits and the DCB Algorithm}
\vspace{-0.1cm}
\textbf{The generic Confidence Ball (CB) algorithm}
is designed for a single agent linear bandit problem (i.e. $|V| = 1$). The algorithm maintains a confidence ball $C_t\subset \bR^d$ within which it believes the true parameter $\theta$ lies with high probability. This confidence ball is computed from the observation pairs, $(x_k,r_k)_{k = 1,\dots,t}$ (for the sake of simplicity, we dropped the agent index, $i$). Typically, the covariance matrix $A_t = \sum_{k=1}^t x_k x_k \tr$ and $b$-vector, $b_t = \sum_{k=1}^t r_k x_k$, are sufficient statistics to characterise this confidence ball. Then, given its current action set, $\cD_t$, the agent selects the optimistic action, assuming that the true parameter sits in $C_t$, i.e. $(x_t,\sim) = \argmax_{(x,\theta')\in\cD_t\times C_t}\{x\tr\theta'\}$. Pseudo-code for CB is given in the Appendix \ref{ap:pseudocode}.

\textbf{Gossip Sharing Protocol for DCB}
We assume that the agents are sharing across a peer to peer network, i.e. every agent can share information with every other agent, but that every agent can communicate with only one other agent per round. In our algorithms, each agent, $i$, needs to maintain
\begin{enumerate}[(1)]
	\item a \emph{buffer} (an ordered set) $\cA_t^i$ of covariance matrices and an \emph{active} covariance matrix $\tA_t^i$,
	\item a \emph{buffer} $\cbB_t^i$ of b-vectors and an \emph{active} $b$-vector $\tb_t^i$,
\end{enumerate}
Initially, we set, for all $i\in V$, $\tA_0^i = I$, $\tb_0^i = 0$.
These \emph{active} objects are used by the algorithm as sufficient statistics from which to calculate confidence balls, and summarise only information gathered before or during time $\tau(t)$, where $\tau$ is an arbitrary monotonically increasing function satisfying $\tau(t)<t$.
The buffers are initially set to $\cA_0^i = \emptyset$, and $\cbB_0^i = \emptyset$.
For each $t>1$, each agent, $i$, shares and updates its buffers as follows:
\begin{enumerate}[(1)]
	\item a random permutation, $\sigma$, of the numbers $1,\dots,|V|$ is chosen uniformly at random in a decentralised manner among the agents,\footnotemark
\footnotetext{This can be achieved in a variety of ways.}
	\item the buffers of $i$ are then updated by averaging its buffers with those of $\sigma(i)$, and then extending them using their current observations\footnotemark
\footnotetext{The $\circ$ symbol denotes the concatenation operation on two ordered sets: if $x = (a,b,c)$ and $y = (d,e,f)$, then $x\circ y = (a,b,c,d,e,f)$, and $y\circ x = (d,e,f,a,b,c)$.}
		\begin{gather*}
			{\cbA_{t+1}^i} = \left(\left( \tfrac{1}{2} ( \cbA_t^{i} + \cbA_t^{\sigma(i)} ) \right)
											\circ \left(x_{t+1}^i \left(x_{t+1}^i\right)\tr\right)\right)
			\text{, }\\
			{\cbB_{t+1}^i} = \left(\left( \tfrac{1}{2} ( \cbB_t^{i} + \cbB_t^{\sigma(i)} ) \right)
											\circ \left(r_{t+1}^i x_{t+1}^i \right)\right),
		\end{gather*}
		$\tA_{t+1}^i = \tA_t^{i} + \tA_t^{\sigma(i)} $, and $\tb_{t+1}^i = \tb_t^{i} + \tb_t^{\sigma(i)} $.
	\item if the length $|\cbA_{t+1}^i|$ exceeds $t-\tau(t)$, the first element of $\cbA^i_{t+1}$ is added to $\tA_{t+1}^i$ and deleted from $\cbA_{t+1}^i$. $\cbB_{t+1}^i$ and $\tb_{t+1}^i$ are treated similarly.
\end{enumerate}
In this way, each buffer remains of size at most $t - \tau(t)$, and contains only information gathered after time $\tau(t)$.
The result is that, after $t$ rounds of sharing, the current covariance matrices and b-vectors used by the algorithm to make decisions have the form:
\begin{gather}\label{eq:current_weights}
\tA_t^i := I + \sum_{t'=1}^{\tau(t)} \sum_{i' = 1}^{|V|}  w_{i,t}^{i',t'} x_{t'}^{i'}{ x_{t'}^{i'}}\tr,\nonumber\\
\text{ and }
\tb_t^i := \sum_{t'=1}^{\tau(t)} \sum_{i' = 1}^{|V|}  w_{i,t}^{i',t'} r_{t'}^{i'} x_{t'}^{i'}.\nonumber
\end{gather}
where the weights $w_{i,t}^{i',t'}$ are random variables which are unknown to the algorithm.
Importantly for our analysis, as a result of the overlay protocol's uniformly random choice of $\sigma$, they are identically distributed ($i.d.$) for each fixed pair $(t,t')$, and $\sum_{i'\in V} w_{i,t}^{i',t'} = |V|$.
If information sharing was perfect at each time step, then the current covariance matrix could be computed using all the information gathered by all the agents, and would be:
\begin{align}
A_t := I + \sum_{i' = 1}^{|V|} \sum_{t'=1}^{t} x_{t'}^{i'}\left( x_{t'}^{i'}\right)\tr.
\label{eq: def of At}
\end{align}

\vspace{-0.1cm}
\textbf{DCB algorithm}
The OFUL algorithm \cite{abbasi2011improved} is an improvement of the confidence ball algorithm from \cite{dani2008stochastic}, which assumes that the confidence balls $C_t$ can be characterised by $A_t$ and $b_t$.
In the DCB algorithm, each agent $i\in V$ maintains a confidence ball $C_t^i$ for the unknown parameter $\theta$ as in the OFUL algorithm, but calculated from $\tA_t^i$ and $\tb_t^i$.
It then chooses its action, $x_t^i$, to satisfy
$(x_t^i, \theta_t^i) = \argmax_{(x,\theta)\in\cD_t^i \times C_t^i} x\tr \theta$,
and receives a reward $r_t^i$.
Finally, it shares its information buffer according to the sharing protocol above.
Pseudo-code for DCB is given in Appendix \ref{ap:pseudocode}, and in Algorithm \ref{alg:DCCB}.

\vspace{-0.1cm}
\subsection{Results for DCB}\label{sec:DCB}
\begin{theorem}\label{thm:Main}
Let $\tau(\cdot):t\rightarrow 4\log(|V|^{\frac{3}{2}}t)$. Then, with probability $1-\delta$, the regret of DCB is bounded by
\begin{align*}
\cR_t
	\le&  \left(N(\delta)|V| + \nu(|V|,d,t)\right)\|\theta\|_2\\
 						&+ 4e^2 \left(\beta(t) + 4R \right)
 						\sqrt{|V|t \ln\left( \left(1+|V|t/d\right)^d\right) },
\end{align*}
where $\nu(|V|,d,t):=(d+1)d^2 (4|V|\ln(|V|^{\frac{3}{2}}t))^3$, $N(\delta):= \sqrt{3}/((1 - 2^{-\frac{1}{4}})\sqrt{\delta})$, and 
\begin{gather}
\label{eq: def of beta}
\beta(t)
  :=R \sqrt{\ln \left(\frac{\left(1 + |V|t/d \right)^d}{\delta}\right)}
		 						 + \| \theta\|_2.
\end{gather}
\end{theorem}
The term $\nu(t,|V|,d)$ describes the loss compared to the centralised algorithm due to the delay in using information, while $N(\delta)|V|$ describes the loss due to the incomplete mixing of the data across the network.

If the agents implement CB independently and do not share any information, which we call CB-\emph{NoSharing},
then it follows from the results in \cite{abbasi2011improved}, the equivalent regret bound would be
\begin{align}
\cR_t
	\le &|V| \beta(t)\sqrt{t \ln\left( (1+t/d)^d\right) }\label{eq:regret_CB}
\end{align}
Comparing Theorem \ref{thm:Main} with \eqref{eq:regret_CB} tells us that, after an initial ``burn in" period, the gain in regret performance of DCB over CB-\emph{NoSharing} is of order almost $\sqrt{|V|}$.

\begin{corollary}
We can recover a bound in expectation from Theorem \ref{thm:Main}, by using the value $\delta = 1/\sqrt{|V|t}$:
\begin{align*}
&\bE[\cR_t]
	\le O( t^{\frac{1}{4}}) + \sqrt{|V|t}\|\theta\|_2 \\
		  &+ 4e^2\left(  R\sqrt{\ln \left(\left(1 + |V|t/d \right)^d\sqrt{|V|t} \right)}
		 						 + \| \theta\|_2 + 4R \right)\\
		 &\hspace{3cm}\times\sqrt{|V|t \ln\left( (1+|V|t/d)^d\right) }.
\end{align*}
\end{corollary}
\vspace{-0.2cm}  
This shows that DCB exhibits asymptotically optimal regret performance, up to log factors, in comparison with any algorithm that can share its information perfectly between agents at each round.

\subsubsection*{Communication Complexity}
If the agents communicate their information to each other at each round without a central server, then every agent would need to communicate their chosen action and reward to every other agent at each round, giving a communication cost of  order $d|V|^2$ per-round.
We call such an algorithm CB-\emph{InstSharing}.
Under the gossip protocol we propose each agent requires at most $O(log_2(|V|t)d^2|V|)$ bits to be communicated per round.
Therefore, a significant communication cost reduction is gained when $log(|V|t)d \ll |V|$.

Using an epoch-based approach, as in \cite{szorenyi2013gossip}, the per-round communication cost of the gossip protocol becomes $O(d^2|V|)$.
This improves efficiency over any horizon, requiring only that $d \ll |V|$, and the proofs of the regret performance are simple modifications of those for DCB.
However, in comparison with growing buffers this is only an issue after $O(\exp(|V|))$ number of rounds, and typically $|V|$ is large.

While the DCB has a clear communication advantage over CB-\emph{InstSharing}, there are other potential approaches to this problem. For example, instead of randomised neighbour sharing one can use a deterministic protocol such as \emph{Round-Robin} (RR), which can have the same low communication costs as DCB. However, the regret bound for RR suffers from a naturally larger delay in the network than DCB. Moreover, attempting to track potential doubling of data points when using a gossip protocol, instead of employing a delay, leads back to a communication cost of order $|V|^2$ per round. More detail is included in Appendix \ref{ap:complexity}.

\subsubsection*{Proof of Theorem \ref{thm:Main}}
In the analysis we show that the bias introduced by imperfect information sharing is mitigated by delaying the inclusion of the data in the estimation of the parameter $\theta$.
The proof builds on the analysis in \cite{abbasi2011improved}.
The emphasis here is to show how to handle the extra difficulty stemming from imperfect information sharing, which results in the influence of the various rewards at the various peers being unbalanced and appearing with a random delay.
Proofs of the Lemmas \ref{lem:cov-matrix-det} and \ref{lem:weight-norm-sum-bound}, and of Proposition \ref{prop:concentration} are crucial, but technical, and are deferred to Appendix \ref{ap:intermediary}.

\textbf{Step 1: Define modified confidence ellipsoids. }
First we need a version of the confidence ellipsoid theorem given in \cite{abbasi2011improved} that incorporates the bias introduced by the random weights:
\begin{proposition}\label{prop:concentration}
Let $\delta>0$,  $\ttheta_t^i := (\tA_t^i)^{-1}\tb_t^i$, $W(\tau) := \max\{w_{i,t}^{i',t'}:t,t'\le \tau, \; i,i' \in V\}$, and let
\begin{align}\label{eq:ellipsoid_1}
C_t^i
	:= \bigg\{ x \in \bR^d :& \| \ttheta_t^i - x \|_{\tA_t^i}
											\le \| \theta\|_2 \\
											&+ W(\tau(t))R
															\sqrt{2\log\left(\det(\tA_t^i)^{\frac{1}{2}}
																							   /\delta\right)}
					\bigg\}.\nonumber
\end{align}
Then with probability $1-\delta$, $\theta\in C_t^i$.
\end{proposition}
\vspace{-0.25cm}  
In the rest of the proof we assume that $\theta\in C_t^i$.

\textbf{Step 2: Instantaneous regret decomposition.} Denote by $(x_t^i,\theta_t^i) = \argmax_{x\in D_t^i, y \in C_t^i} x\tr y$. Then we can decompose the instantaneous regret, following a classic argument (see the proof of Theorem 3 in \cite{abbasi2011improved}):
\begin{align}
\rho_t^i &=  \left(x_t^{i,\ast}\right)\tr \theta -  (x_t^i) \tr \theta 
 \le \left(x_t^i\right)\tr \theta_t^i -  (x_t^i) \tr \theta \notag\\
 & = \left(x_t^i\right)\tr\left[ \left( \theta_t^i -  \ttheta_t^i\right) 
		+ \left( \ttheta_t^i -  \theta \right) \right] \notag\\
  &\le \| x_t^i \|_{\left(\tA_t^i\right)^{-1}}
  		\left[ \left\| \theta_t^i -  \ttheta_t^i\right \|_{\tA_t^i}
		+ \left\| \ttheta_t^i -  \theta \right\|_{\tA_t^i} \right]
\label{eq: inst regret decomp}
\end{align}

\textbf{Step 3: Control the bias.}
The norm differences inside the square brackets of the regret decomposition are bounded through \eqref{eq:ellipsoid_1} in terms of the matrices $\tA_t^i$.
We would like, instead, to have the regret decomposition in terms of the matrix $A_t$ (which is defined in \eqref{eq: def of At}).
To this end, we give some lemmas showing that using the matrices $\tA_t^i$ is almost the same as using $A_t$.
These lemmas involve elementary matrix analysis, but are crucial for understanding the impact of imperfect information sharing on the final regret bounds.

\textbf{Step 3a: Control the bias coming from the weight imbalance.}
\begin{lemma}[Bound on the influence of general weights]\label{lem:cov-matrix-det}
For all $i\in V$ and $t>0$,
\begin{gather*}
\| x_t^i \|_{\left(\tA_t^i\right)^{-1}}^2
						\le e^{ \sum_{t'=1}^{\tau(t)}  \sum_{i' = 1}^{|V|} \left| w_{i,t}^{i',t'} - 1 \right|  }
																		\| x_t^i \|_{\left(A_{\tau(t)}\right)^{-1}}^2,\\
\text{and }																	
\det \left(\tA_t^i \right)
					\le e^{ \sum_{t'=1}^{\tau(t)}  \sum_{i' = 1}^{|V|} \left| w_{i,t}^{i',t'} - 1 \right| }
																		\det\left(A_{\tau(t)} \right) \nonumber.
\end{gather*}
\end{lemma}

Using Lemma 4 in \cite{szorenyi2013gossip}, by exploiting the random weights are identically distributed ($i.d.$) for each fixed pair $(t,t')$, and $\sum_{i'\in V} w_{i,t}^{i',t'} = |V|$ under our gossip protocol, we can control the random exponential constant in Lemma \ref{lem:cov-matrix-det}, and the upper bound $W(T)$
using the Chernoff-Hoeffding bound:
\begin{lemma}[Bound on the influence of weights under our sharing protocol]\label{lem:weight-norm-sum-bound}
Fix some constants $0<\delta_{t'}<1$. Then with probability $1-\sum_{t' = 1}^{\tau(t)} \delta_{t'}$
\begin{gather*}
\sum_{i' = 1}^{|V|} \sum_{t'=1}^{\tau(t)} \left| w_{i,t}^{i',t'} - 1 \right|
	\le |V|^{\frac{3}{2}} \sum_{t'=1}^{\tau(t)} \left(2^{(t-t')}\delta_{t'}\right)^{-\frac{1}{2}},\\
\text{ and }
W(T)
	\le 1+ \max_{1\le t' \le {\tau(t)}} \left\{ |V|^{\frac{3}{2}}\left(2^{(t-t')}\delta_{t'}\right)^{-\frac{1}{2}}\right\}.
\end{gather*}
\end{lemma}

In particular, for any $\delta \in (0,1)$, choosing $\delta_{t'} =  \delta 2^{\tfrac{t'-t}{2}}$,
with probability 
$1-\delta/(|V|^{3}t^{2}(1-2^{-1/2}))$ we have
\begin{gather}\label{eq:A_bound}
\sum_{i' = 1}^{|V|} \sum_{t'=1}^{\tau(t)} \left| w_{i,t}^{i',t'} - 1 \right|
	\le \frac{1}{(1 - 2^{-\frac{1}{4}})t\sqrt{\delta}},\nonumber\\
\text{and }
W(\tau(t))
	\le 1+ \frac{|V|^{\frac{3}{2}}}{t\sqrt{\delta}}.
\end{gather}
Thus Lemma \ref{lem:cov-matrix-det} and \ref{lem:weight-norm-sum-bound} give us control over the bias introduced by the imperfect information sharing.
Combining them with Equations (\ref{eq:ellipsoid_1}) and (\ref{eq: inst regret decomp}) we find that with probability
$1-\delta/(|V|^{3}t^{2}(1-2^{-1/2}))$:
\begin{align}\label{eq:sr_high_prob}
\rho_t^i
	\le& 2e^{C(t)}
			\| x_t^i \|_{\left(A_{\tau(t)}^i \right)^{-1}}
  		 			\left(1+ C(t)\right)\\
  					&\times\left[R\sqrt{2\log\left(e^{C(t)}
  														\det\left(A_{\tau(t)} \right)^{\frac{1}{2}}
  								    							\delta^{-1}
  								    					\right)
  								    	}
  							+ \| \theta\|
  							\right]\nonumber
\end{align}
where 
$C(t) := 1/(1-2^{-1/4})t\sqrt{\delta}$

\textbf{Step 3b: Control the bias coming from the delay.} Next, we need to control the bias introduced from leaving out the last $4\log(|V|^{3/2} t)$ time steps from the confidence ball estimation calculation:
\begin{proposition}\label{prop:num-outliers-delayed}
There can be at most
\begin{align}\label{eq:no-of-outliers}
\nu(k) := (4|V|\log(|V|^{3/2} k))^3(d+1)d(tr(A_0)+1)
\end{align}
pairs $(i,k)\in{1,\dots,|V|}\times\{1,\dots,t\}$ for which one of 
\begin{gather*}
\|x_k^i\|^2_{ A_{\tau(k)}^{-1}}
	\ge e \|x_k^i\|^2_{ \left(A_{k-1} + \sum_{j=1}^{i-1}x_{k}^{j}(x_{k}^{j})\tr \right)^{-1}},\\
\text{or }
\det \left(A_{\tau(k)} \right)
	\ge e \det \left(A_{k-1} + \sum_{j=1}^{i-1} x_{k}^{j}(x_{k}^{j})\tr\right)
\text{ holds.}
\end{gather*}
\end{proposition}

\textbf{Step 4: Choose constants and sum the simple regret.}
Defining a constant
$$
N(\delta):=\frac{1}{(1 - 2^{-\frac{1}{4}})\sqrt{\delta}},
$$
we have, for all $k\ge N(\delta)$,
$C(k)\leq 1$, and so, by (\ref{eq:sr_high_prob}) with probability $1-(|V|k)^{-2}\delta/(1-2^{-1/2})$
\begin{align}\label{eq:simple-regret}
\rho_k^i \le &2e\| x_k^i \|_{A_{\tau(k)}^{-1}}\\
  					 		&\times\left[2R\sqrt{2\log \left(\frac{e\det\left(A_{\tau(k)} \right)^{\frac{1}{2}}}
  					 														{\delta}
  					 								\right)
  					 							}
  					 				+\| \theta\|_2
  					 		\right].\nonumber
\end{align}
Now, first applying Cauchy-Schwarz, then step 3b from above together with \eqref{eq:simple-regret}, and finally Lemma 11 from \cite{abbasi2011improved} yields that, with probability $1 - \left(1 + \sum_{t = 1}^\infty (|V|t)^{-2}/(1-2^{-1/2})\right) \delta\ge 1 - 3\delta$,
\begin{align*}
\cR_t \le& N(\delta) |V| \|\theta\|_2
						+ \left[
									|V|t\sum_{t' = N(\delta)}^t\sum_{i = 1}^{|V|}\left(\rho_{t'}^i\right)^2
						    \right]^{\frac{1}{2}}\\
		\le& \left(N(\delta)|V| + \nu(|V|,d,t)\right)\|\theta\|_2
						\\&+  4e^2 \left(\beta(t) + 2R \right)
						\left[
								|V|t\sum_{t' = 1}^t\sum_{i = 1}^M  \| x_t^i \|_{\left(A_t \right)^{-1}}^2
						 \right]^{\frac{1}{2}} \\
 		\le& \left(N(\delta)|V| + \nu(|V|,d,t)\right)\|\theta\|_2
 						\\&+ 4e^2 \left(\beta(t) + 2R \right)
 						\sqrt{|V|t \left(2\log\left( \det\left(A_t \right)\right)\right) },
\end{align*}
where $\beta(\cdot)$ is as defined in \eqref{eq: def of beta}.
Replacing $\delta$ with $\delta/3$ finishes the proof.

\subsubsection*{Proof of Proposition \ref{prop:num-outliers-delayed}}
This proof forms the major innovation in the proof of Theorem \ref{thm:Main}.
Let $(y_k)_{k\ge1}$ be any sequence of vectors such that $\|y_k\|_2\le 1$ for all $k$, and let $B_n := B_0 + \sum_{k = 1}^{n} y_k y_k\tr$, where $B_0$ is some positive definite matrix.

\begin{lemma}\label{lem:norm-transf-2-modified}
For all $t>0$, and for any $c\in (0,1)$, we have
\begin{align*}
 &\left| \left\{ k \in \{1,2,\dots\}: \|y_k\|_{B_{k-1}^{-1}}^2 > c \right\}\right| \nonumber\\
 	&\qquad\qquad\qquad\leq  (d+c)d(tr(B_0^{-1})-c)/c^2,
\end{align*}
\end{lemma}

\begin{proof}
We begin by showing that, for any $c\in(0,1)$
\begin{align}\label{eq:assump-delayed}
\|y_k\|_{B_{k-1}^{-1}}^2 >  c
\end{align}
can be true for only $2dc^{-3}$ different $k$.

Indeed, let us suppose that \eqref{eq:assump-delayed} is true for some $k$. Let $(e_i^{(k-1)})_{1\le i \le d}$ be the orthonormal eigenbasis for $B_{k-1}$, and, therefore, also for $B_{k-1}^{-1}$, and write $y_k = \sum_{i = 1}^d\alpha_i e_i$. Let, also, $(\lambda_i^{(k-1)})$ be the eigenvalues for $B_{k-1}$. Then,
\begin{gather*}
 c
 	< y_k\tr B_{k-1}^{-1} y_k
 	= \sum_{i = 1}^d \tfrac{\alpha_i^2}{\lambda_i^{(k-1)}}\leq  tr(B_{k-1}^{-1} ),\\
 \implies\ 
  \exists j\in\{1,\dots,d\}:\tfrac{\alpha_j^2}{\lambda_j^{(k-1)}},\ \tfrac{1}{\lambda_j^{(k-1)}} >  \tfrac{c}{d},
\end{gather*}
where we have used that $\alpha_i^2<1$ for all $i$, since $\|y_k\|_2<1$. Now,
\begin{align*}
tr&(B_{k-1}^{-1}) - tr(B_k^{-1})\\
 &= tr(B_{k-1}^{-1}) - tr((B_{k-1}+y_k y_k\tr)^{-1})\\
 &>  tr(B_{k-1}^{-1})  - tr((B_{k-1}+\alpha_j^2e_j e_j\tr)^{-1})\\
 &= \tfrac{1}{\lambda_j^{(k-1)}} - \tfrac{1}{\lambda_j^{(k-1)} + \alpha_j^2}
 = \tfrac{\alpha_j^2}{\lambda_j^{(k-1)}(\lambda_j^{(k-1)} + \alpha_j^2)}\\
 &> \left(d^2c^{-2} + dc^{-1}\right)^{-1}
 > \tfrac{c^2}{d(d+c)}
\end{align*}
So we have shown that \eqref{eq:assump-delayed} implies that
\begin{gather*}
	tr(B_{k-1}^{-1} )>  c
	\text{ and }
	\ tr(B_{k-1}^{-1}) - tr(B_k^{-1}) > \frac{c^2}{d(d+c)}.
\end{gather*}
Since $tr(B_0^{-1})\ge tr(B_{k-1}^{-1}) \ge tr(B_k^{-1})\ge 0$ for all $k$, it follows that \eqref{eq:assump-delayed} can be true for at most $(d+c)d(tr(B_0^{-1})-c)c^{-2}$ different $k$.
\end{proof}

Now, using an argument similar to the proof of Lemma~\ref{lem:cov-matrix-det}, for all $k<t$
\begin{gather*}
\|y_{k+1}\|_{B_{\tau(k)}^{-1}}
	\le e^{\sum_{s = \tau(k)+1}^{k}\|y_{s+1}\|_{B_{s}^{-1}}}
																		\|y_{k+1}\|_{B_{k}^{-1}},\\
\text{ and }
\det \left(B_{\tau(t)} \right)
	\le e^{ \sum_{k = \tau(t)+1}^{t} \| y_k \|_{B_k^{-1}}^2}
																		\det\left(B_t \right).
\end{gather*}
Therefore,
\begin{gather*}
\ \|y_{k+1}\|_{B_{\tau(k)}^{-1}} \geq c \|y_{k+1}\|_{B_k^{-1}} 
\text{ or } 
\det(B_{\tau(k)}) \geq c \det(B_k)\\
\implies
\sum_{s=\tau(k)}^{k-1} \|y_{s+1}\|_{B_s^{-1}} \geq \ln(c)
\end{gather*}
However, according to Lemma \ref{lem:norm-transf-2-modified}, there can be at most 
\begin{align*}
\nu(t)	:= \left(d+\tfrac{\ln(c)}{\Delta(t)}\right)d\left(tr\left(B_0^{-1}\right)-\tfrac{\ln(c)}{\Delta(t)}\right)
			\left(\tfrac{\Delta(t)}{\ln(c)}\right)^2
\end{align*}
times $s\in \{1,\dots,t\}$, such that $\|y_{s+1}\|_{B_{s}^{-1}}\ge \ln(c)/\Delta(t)$, where $\Delta(t) := \max_{1\le k\le t}\{k - \tau(k)\}$.
Hence
$\sum_{s=\tau(j)+1}^{k}\|y_{s+1}\|_{B_s}^{-1} \geq \ln(c)$
is true for at most $\Delta(t)\nu(|V|,d,t)$ indices $k\in\{1,\dots,t\}$.

Finally, we finish by setting $(y_k)_{k\ge 1} = \circ_{t\ge1}(x_t^i)_{i=1}^{|V|}$.

\section{Clustering and the DCCB Algorithm}
We now incorporate distributed clustering into the DCB algorithm. The analysis of DCB forms the backbone of the analysis of DCCB.
\begin{algorithm}[ht!]
\caption{Distributed {\color{blue}Clustering} Confidence Ball}
\begin{algorithmic} \label{alg:DCCB}
\STATE {\bfseries Input:} Size of network $|V|, \tau:t \rightarrow t - 4\log_2{t}, \alpha, \lambda$\\
{\bfseries Initialization:} $\forall i\in V$, set $\tA_0^i = I_d$, $\tb_0^i = \mathbf{0}$, $\cbA_0^i = \cbB_0^i = \emptyset$, and $V_0^i = V$.\\
\FOR{$t=0, \dots \infty$}
	\STATE Draw a random permutation $\sigma$ of $\{1,\dots,V\}$ respecting the current local clusters
	\FOR{$i = 1, \dots, |V|$}
		\STATE Receive action set $\cD_t^i$ and construct the confidence ball $C_t^i$ using $\tA_t^i$ and $\tb_t^i$
	\STATE \textbf{Choose action and receive reward:}
		\STATE Find $(x_{t+1}^i,\ast) = \argmax_{(x,\ttheta)\in\cD_t^i\times C_t^i} x\tr\ttheta$, and get reward $r_{t+1}^i$ from context $x_{t+1}^i$.
	\STATE \textbf{Share and update information buffers:}
		\hspace*{-\fboxsep}\colorbox{blue!10}{\parbox{\linewidth}{\color{blue}
			\STATE {\bfseries if }$\| \hattheta_{local}^i - \hattheta_{local}^{j}\| >c_{\lambda}^{thresh}(t)$}}
			\hspace*{-\fboxsep}\colorbox{blue!10}{\parbox{\linewidth}{\color{blue}
				\STATE  Update local cluster: $V_{t+1}^i = V_t^i\setminus \{\sigma(i)\}$, $V_{t+1}^{{\sigma(i)}} = V_{t}^{\sigma(i)}\setminus \{i\}$, and reset according to \eqref{eq:reseting}}}
			\hspace*{-\fboxsep}\colorbox{blue!10}{\parbox{\linewidth}{\color{blue}
			\STATE {\bfseries elseif } $V_t^i = V_t^{\sigma(i)}$}}
			\STATE Set $\cbA_{t+1}^i = \left( \frac{1}{2} ( \cbA_t^{i} + \cbA_t^{\sigma(i)} ) \right) \circ (x_{t+1}^i \left(x_{t+1}^i\right)\tr)$ and $\cbB_{t+1}^i= \left( \frac{1}{2} ( \cbB_t^{i} + \cbB_t^{\sigma(i)} ) \right) \circ (r_{t+1}^i x_{t+1}^i )$	
			\hspace*{-\fboxsep}\colorbox{blue!10}{\parbox{\linewidth}{\color{blue}
			\STATE {\bfseries else }  Update: Set $\cbA_{t+1}^i  = \cbA_t^i  \circ (x_{t+1}^i \left(x_{t+1}^i\right)\tr)$ and $\cbB_{t+1}^i  = \cbB_t^i   \circ (r_{t+1}^i x_{t+1}^i )$}}
			\hspace*{-\fboxsep}\colorbox{blue!10}{\parbox{\linewidth}{\color{blue}
			\STATE {\bfseries endif}}}
			\hspace*{-\fboxsep}\colorbox{blue!10}{\parbox{\linewidth}{\color{blue}
			\STATE Update local estimator: $A_{local,t+1}^i = A_{local,t}^i + x_{t+1}^i \left(x_{t+1}^i\right)\tr$,  $b_{local,t+1}^i = b_{local,t}^i + r_{t+1}^i x_{t+1}^i$, and $\hattheta_{local,t+1} =\left( A_{local,t+1}^i \right)^{-1}b_{local,t+1}^i$}}
			\hspace*{-\fboxsep}\colorbox{blue!10}{\parbox{\linewidth}{\color{blue}
			\STATE {\bfseries if} $|\cbA_{t+1}^i| > t - \tau(t) $ set $\tA_{t+1}^i = \tA_t^i + \cbA_{t+1}^i(1)$, $\cbA_{t+1}^i = \cbA_{t+1}^i\setminus \cbA_{t+1}^i(1)$. Similarly for $\cbB_{t+1}^i$.}}
	\ENDFOR
\ENDFOR
\end{algorithmic}
\end{algorithm}

\textbf{DCCB Pruning Protocol}
In order to run DCCB, each agent $i$ must maintain some local information buffers in addition to those used for DCB. These are:
\begin{enumerate}[(1)]
	\item a local covariance matrix $A_{local}^i = A_{local,t}^i$, a local b-vector $b_{local}^i = b_{local,t}^i$,
	\item and a local neighbour set $V_t^i$.
\end{enumerate}
The local covariance matrix and b-vector are updated as if the agent was applying the generic (single agent) confidence ball algorithm: $A_{local,0}^i = A_0$, $b_{local,0}^i = 0$,
\begin{gather*}
A_{local,t}^i = x_t^i(x_t^i)\tr + A_{local,t-1}^i,\\
\text{ and } b_{local,t}^i = r_t^i x_t^i + b_{local,t-1}^i.
\end{gather*}

\textbf{DCCB Algorithm}
Each agent's local neighbour set $V_t^i$ is initially set to $V$. At each time step $t$, agent $i$ contacts one other agent, $j$, at random from $V_t^i$, and both decide whether they do or do not belong to the same cluster. To do this they share local estimates, $\hat\theta_t^i = {A_{local,t}^i }^{-1}b_{local,t}^i $ and $\hat\theta_t^j = {A_{local,t}^j }^{-1}b_{local,t}^j $, of the unknown parameter of the bandit problem they are solving, and see if they are further apart than a threshold function $c = c_{\lambda}^{thresh}(t)$, so that if
\begin{align}\label{eq:thresh}
\| \hat\theta_t^i -\hat\theta_t^j \|_2 \ge  c_{\lambda}^{thresh}(t),
\end{align}
then $V_{t+1}^i = V_t^i \setminus \{j\}$ and $V_{t+1}^j = V_t^j \setminus \{i\}$.
Here $\lambda$ is a parameter of an extra assumption that  is needed, as in \cite{gentile2014online}, about the process generating the context sets $\cD^i_t$:
\begin{description}
	\item [(A)] Each context set $\cD^i_t = \{x_k\}_{k}$ is finite and contains $i.i.d.$ random vectors such that for all, $k$, $\|x_k\| \le 1$ and $\bE(x_k x_k\tr)$ is full rank, with minimal eigenvalue $\lambda>0$.
\end{description}
We define $c_{\lambda}^{thresh}(t)$, as in \cite{gentile2014online}, by
\begin{align}\label{eq:thresh_def}
c_{\lambda}^{thresh}(t):= \frac{R\sqrt{2d \log(t) + 2\log(2/\delta)}+1}
											{\sqrt{1 +  \max\left\{A_\lambda(t,\delta/(4d)), 0 \right\}}}
\end{align}
where 
$
A_\lambda(t,\delta):=\tfrac{\lambda t}{\delta} - 8\log\tfrac{t+3}{\delta}
										    	 - 2\sqrt{t\log\tfrac{t+3}{\delta}}.
$

The DCCB algorithm is pretty much the same as the DCB algorithm, except that it also applies the pruning protocol described. In particular, each agent, $i$, when sharing its information with another, $j$, has three possible actions:
\vspace{-0.3cm}
\begin{enumerate}[(1)]
	\item if \eqref{eq:thresh} is not satisfied and $V_t^i = V_t^j$, then the agents share simply as in the DCB algorithm;
	\item if \eqref{eq:thresh} is not satisfied but $V_t^i \neq V_t^j$, then no sharing or pruning occurs.
	\item if \eqref{eq:thresh} is satisfied, then both agents remove each other from their neighbour sets and reset their buffers and active matrices so that
	\vspace{-0.15cm}
	\begin{gather}
		\cbA^i = (0,0,\dots,A_{local}^i), \cbB^i = (0,0,\dots,b_{local}^i), \nonumber\\
		\text{and } \tA^i = A_{local}^i, \tb^i = b_{local}^i,\label{eq:reseting}
	\end{gather}
	and similarly for agent $j$.
\end{enumerate}
\vspace{-0.3cm}
It is proved in the theorem below, that under this sharing and pruning mechanism, in high probability after some finite time each agent $i$ finds its true cluster, i.e. $V_t^i = U^k$. Moreover, since the algorithm resets to its local information each time a pruning occurs, once the true clusters have been identified, each cluster shares only information gathered within that cluster, thus avoiding introducing a bias by sharing information gathered from outside the cluster before the clustering has been identified. Full pseudo-code for the DCCB algorithm is given in Algorithm \ref{alg:DCCB}, and the differences with the DCB algorithm are highlighted in blue.

\begin{figure}[H]
\includegraphics[width = 0.5\textwidth]{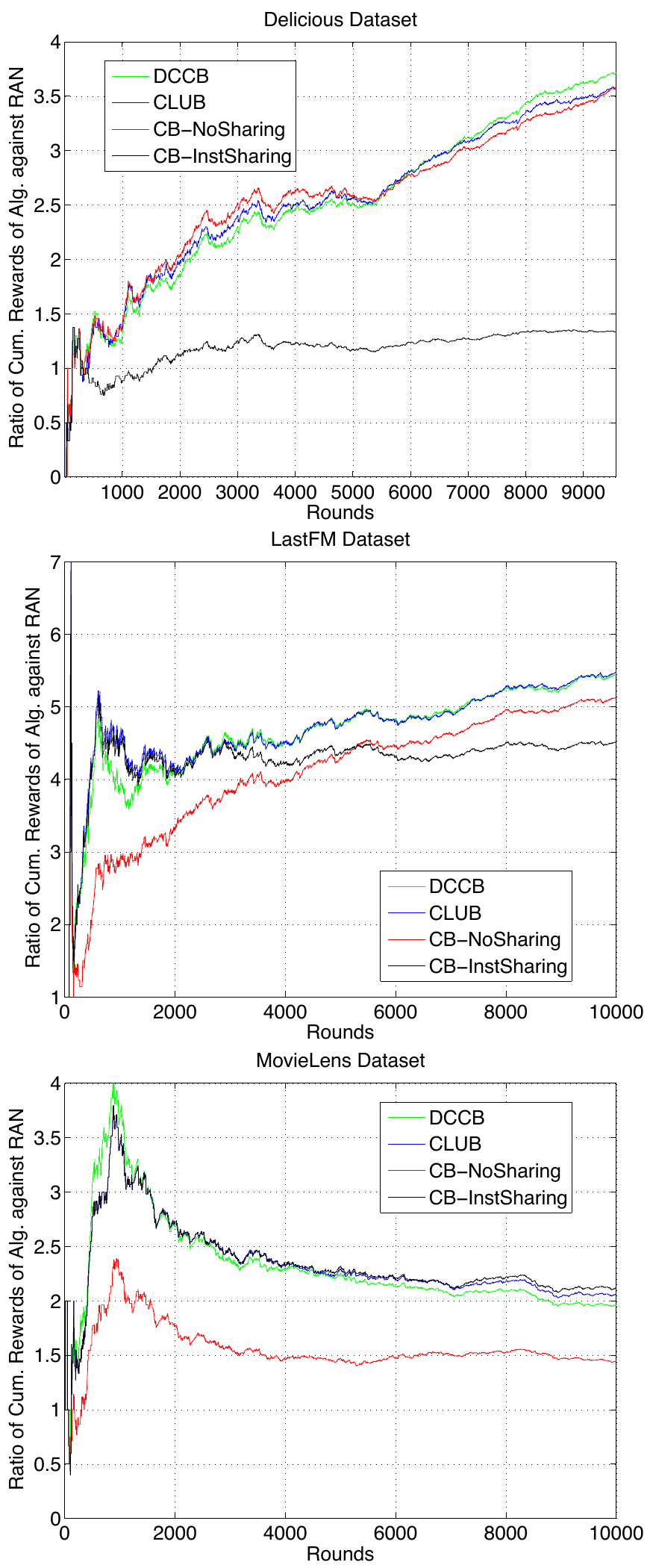}
\vspace{-0.7cm}
\caption{Here we plot the performance of DCCB in comparison to CLUB, CB-\emph{NoSharing} and CB-\emph{InstSharing}. The plots show the ratio of cumulative rewards achieved by the algorithms to the cumulative rewards achieved by the random algorithm.}
\label{fig_exp_res}
\end{figure}
\subsection{Results for DCCB}

\begin{theorem}\label{thm:DCCB_main}
Assume that (A) holds, and let $\gamma$ denote the smallest distance between the bandit parameters $\theta^k$. Then there exists a constant $C = C(\gamma,|V|,\lambda,\delta)$, such that with probability $1-\delta$ the total cumulative regret of cluster $k$ when the agents employ DCCB is bounded by
\begin{align*}
\cR_t
	\le&  \bigg[\max\left\{\sqrt{2}N(\delta), C + 4\log_2 (|V|^{\frac{3}{2}}C)\right\}|U^k|\\
							&\hspace{4.2cm}+\nu(|U^k|,d,t) \bigg] \|\theta\|_2\\
		 &+ 4e\left(\beta(t) + 3R \right)
		  	  \sqrt{ |U^k|t \ln\left( \left(1+|U^k| t/d\right)^d\right)  },
\end{align*}
where $N$ and $\nu$ are as defined in Theorem \ref{thm:Main}, and  $
\beta(t)
	:= R\sqrt{2\ln\left( \left(1+|U^k| t/d\right)^d\right)}
		+ \|\theta\|_2.
$
\end{theorem}
The constant $C(\gamma,|V|,\lambda,\delta)$ is the time that you have to wait for the true clustering to have been identified, 

The analysis follows the following scheme: When the true clusters have been correctly identified by all nodes, within each cluster the algorithm, and thus the analysis, reduces to the case of Section \ref{sec:DCB}. We adapt results from \cite{gentile2014online} to show how long it will be before the true clusters are identified, in high probability. The proof is deferred to Appendices \ref{ap:DCCB_proof} and \ref{ap:DCCB_intermediary}.

\section{Experiments and Discussion}
\textbf{Experiments}
We closely implemented the experimental setting and dataset construction principles used in \cite{shuai16gclub,shuai16cofiba}, and for a detailed description of this we refer the reader to \cite{shuai16gclub}.
We evaluated DCCB on three real-world datasets against its centralised counterpart CLUB, and against the benchmarks used therein, CB-\emph{NoSharing}, and CB-\emph{InstSharing}. 
The LastFM dataset comprises of $91$ users, each of which appear at least $95$ times. The Delicious dataset has $87$ users, each of which appear at least $95$ times. The MovieLens dataset contains $100$ users, each of which appears at least $250$ times. 
The performance was measured using the ratio of cumulative reward of each algorithm to that of the predictor which chooses a random action at each time step.
This is plotted in in Figure \ref{fig_exp_res}.
From the experimental results it is clear that DCCB performs comparably to CLUB in practice, and both outperform CB-\emph{NoSharing}, and CB-\emph{InstSharing}.

\textbf{Relationship to existing literature}
There are several strands of research that are relevant and complimentary to this work.
First, there is a large literature on single agent linear bandits, and other more, or less complicated bandit problem settings.
There is already work on distributed approaches to multi-agent, multi-armed bandits, not least \cite{szorenyi2013gossip} which examines $\epsilon$-greedy strategies over a peer to peer network, and provided an initial inspiration for this current work.
The paper \cite{kalathil2014decentralized} examines the extreme case when there is no communication channel across which the agents can communicate, and all communication must be performed through observation of action choices alone.
Another approach to the multi-armed bandit case, \cite{nayyar2015regret}, directly incorporates the communication cost into the regret.

Second, there are several recent advances regarding the state-of-the-art methods for clustering of bandits. 
The work \cite{shuai16gclub} is a faster variant of \cite{gentile2014online} which adopt the strategy of boosted training stage. In \cite{shuai16cofiba} the authors not only cluster the users, but also cluster the items under collaborative filtering case with a sharp regret analysis.

Finally, the paper \cite{tekin2013distributed} treats a setting similar to ours in which agents attempt to solve contextual bandit problems in a distributed setting. They present two algorithms, one of which is a distributed version of the approach taken in \cite{slivkins2014contextual}, and show that they achieve at least as good asymptotic regret performance in the distributed approach as the centralised algorithm achieves.
However, rather than sharing information across a limited communication channel, they allow each agent only to ask another agent to choose their action for them.
This difference in our settings is reflected worse regret bounds, which are of order $\Omega(T^{2/3})$ at best.

\textbf{Discussion}
Our analysis is tailored to adapt proofs from \cite{abbasi2011improved} about generic confidence ball algorithms to a distributed setting. However many of the elements of these proofs, including Propositions \ref{prop:concentration} and \ref{prop:num-outliers-delayed} could be reused to provide similar asymptotic regret guarantees for the distributed versions of other bandit algorithms, e.g., the Thompson sampling algorithms, \cite{agrawal2012thompson,  kaufmann2012thompson, russo2014learning}.


Both DCB and DCCB are synchronous algorithms. The work on distributed computation through gossip algorithms in \cite{boyd2006randomized} could alleviate this issue. The current pruning algorithm for DCCB guarantees that techniques from \cite{szorenyi2013gossip} can be applied to our algorithms. However the results in \cite{boyd2006randomized} are more powerful, and could be used even when the agents only identify a sub-network of the true clustering.

Furthermore, there are other existing interesting algorithms for performing clustering of bandits for recommender systems, such as COFIBA in \cite{shuai16cofiba}.
It would be interesting to understand how general the techniques applied here to CLUB are.

\vfill
\newpage
\newpage

\section*{Acknowledgments}
We would like to thank the anonymous reviewers for their helpful comments. 
We would also like to thank Gergley Neu for very useful discussions. 
NK thanks the support from EPSRC Autonomous Intelligent Systems project EP/I011587. 
SL thanks the support from MIUR, QCRI-HBKU, Amazon Research Grant and Tsinghua University. 
The research leading to these results has received funding from the European Research Council under the European Union's Seventh Framework Programme (FP/2007-2013) / ERC Grant Agreement n. 306638.

\bibliographystyle{icml2016}
\bibliography{DCoCB.bib}

\newpage

\onecolumn

\appendix
\section{Supplementary Material}

\subsection{Pseudocode for the generic CB algorithm and the DCB algorithm}
\label{ap:pseudocode}

\begin{algorithm}[h]
\caption{Confidence Ball}
\begin{algorithmic} \label{alg:CB}
\STATE {\bfseries Initialization:} Set $A_0 = I$ and $b_0 =  0 $.
\FOR{$t=0, \dots \infty$}
		\STATE Receive action set $\cD_t$
		\STATE Construct the confidence ball $C_t$ using $A_t$ and $b_t$
		\STATE \textbf{Choose action and receive reward:}
		\STATE \quad Find $(x_t, \ast) = \argmax_{(x,\ttheta)\in\cD_t\times C_t} x\tr\ttheta$
		\STATE \quad Get reward $r_t^i$ from context $x_t^i$
		\STATE \quad Update $A_{t+1} = A_t + x_t x_t\tr$ and $b_{t+1} = b_t + r_t x_t$ 
\ENDFOR
\end{algorithmic}
\end{algorithm}

\begin{algorithm}[h]
\caption{Distributed Confidence Ball}
\begin{algorithmic} \label{alg:DCB}
\STATE {\bfseries Input:} Network $V$ of agents, the function $\tau: t \rightarrow t - 4\log_2(|V|^{\frac{3}{2}}t)$.\\
{\bfseries Initialization:} For each $i$, set $\tA_0^i = I_d$ and $\tb_0^i = \mathbf{0}$, and the buffers $\cbA_0^i = \emptyset$ and $\cbB_0^i = \emptyset$.\\
\FOR{$t=0, \dots \infty$}
	\STATE Draw a random permutation $\sigma$ of $\{1,\dots,|V|\}$
	\FOR{\textbf{each agent }$i\in V $}
		\STATE Receive action set $\cD_t^i$ and construct the confidence ball $C_t^i$ using $\tA_t^i$ and $\tb_t^i$
		\STATE \textbf{Choose action and receive reward:}
		\STATE \quad Find $(x_{t+1}^i, \ast) = \argmax_{(x,\ttheta)\in\cD_t^i\times C_t^i} x\tr\ttheta$
		\STATE \quad Get reward $r_{t+1}^i$ from context $x_{t+1}^i$.
		\STATE \textbf{Share and update information buffers:} 
		\STATE \quad Set $\cbA_{t+1}^i = \left( \frac{1}{2} ( \cbA_t^{i} + \cbA_t^{\sigma(i)} ) \right) \circ (x_{t+1}^i \left(x_{t+1}^i\right)\tr)$
								and $\cbB_{t+1}^i= \left( \frac{1}{2} ( \cbB_t^{i} + \cbB_t^{\sigma(i)} ) \right) \circ (r_{t+1}^i x_{t+1}^i )$
		\STATE \quad {\bfseries if} $|\cbA_{t+1}^i| > t - \tau(t) $ set $\tA_{t+1}^i = \tA_t^i + \cbA_{t+1}^i(1)$ and $\cbA_{t+1}^i = \cbA_{t+1}^i\setminus \cbA_{t+1}^i(1)$. Similary for $\cbB_{t+1}^i$.
	\ENDFOR
\ENDFOR
\end{algorithmic}
\end{algorithm}

\subsection{More on Communication Complexity}
\label{ap:complexity}

First, recall that if the agents want to communicate their information to each other at each round without a central server, then every agent would need to communicate their chosen action and reward to every other agent at each round, giving a communication cost of $O(d|V|^2)$ bits per-round.
Under DCB each agent requires at most $O(log_2(|V|t)d^2|V|)$ bits to be communicated per round.
Therefore, a significant communication cost reduction is gained when $log(|V|t)d \ll |V|$.

Recall also that using an epoch-based approach, as in \cite{szorenyi2013gossip}, we reduce the per-round communication cost of the gossip-based approach to $O(d^2|V|)$.
This makes the algorithm more efficient over any time horizon, requiring only that $d \ll |V|$, and the proofs of the regret performance are simple modifications of the proofs for DCB.
In comparison with growing buffers this is only an issue after $O(\exp(|V|))$ number of rounds, and typically $|V|$ is large.
This is why we choose to exhibit the growing-buffer approach in this current work.

Instead of relying on the combination of the diffusion and a delay to handle the potential doubling of data points under the randomised gossip protocol, we could attempt to keep track which observations have been shared with which agents, and thus simply stop the doubling from occurring.
However, the per-round communication complexity of this is at least quadratic in $|V|$, whereas our approach is linear.
The reason for the former is that in order to be efficient, any agent $j$, when sending information to an agent $i$, needs to know for each $k$ which are the latest observations gathered by agent $k$ that agent $i$ already knows about.
The communication cost of this is of order $|V|$.
Since every agent shares information with somebody in each round, this gives per round communication complexity of order $|V|^2$ in the network.

A simple, alternative approach to the gossip protocol is a \emph{Round-Robin} (RR) protocol, in which each agent passes the information it has gathered in previous rounds to the next agent in a pre-defined permutation.
Implementing a RR protocol leads to the agents performing a distributed version of the CB-\emph{InstSharing} algorithm, but with a delay that is of size at least linear in $|V|$, rather than the logarithmic dependence on this quantity that a gossip protocol achieves.
Indeed, at any time, each agent will be lacking $|V|(|V|-1)/2$ observations.
Using this observation, a cumulative regret bound can be achieved using Proposition \ref{prop:num-outliers-delayed} which arrives at the same asymptotic dependence on $|V|$ as our gossip protocol, but with an additive constant that is worse by a multiplicative factor of $|V|$.
This makes a difference to the performance of the network when $|V|$ is very large.
Moreover, RR protocols do not offer the simple generalisability and robustness that gossip protocols offer.

Note that the pruning protocol for DCCB only requires sharing the estimated $\theta$-vectors between agents, and adds at most $O(d|V|)$ to the communication cost of the algorithm. Hence the per-round communication cost of DCCB remains $O(log_2(|V|t)d^2|V|)$.

\begin{figure}[h!]\label{fig:table}
\begin{center}
\begin{tabular}{ c || c | c}
  Algorithm & Regret Bound & Per-Round Communication Complexity\\
  \hline
  CB-\emph{NoSharing} & $O(|V|\sqrt{t})$ & $0$ \\
  CB-\emph{InstSharing} & $O(\sqrt{|V|t})$ & $O(d|V|^2)$ \\
  DCB & $O(\sqrt{|V|t})$ & $O(log_2(|V|t)d^2|V|)$ \\
  DCCB & $O(\sqrt{|U^k|t})$ & $O(log_2(|V|t)d^2|V|)$ \\
\end{tabular}
\end{center}
\caption{This table gives a summary of theoretical results for the multi-agent linear bandit problem. Note that CB with no sharing cannot benefit from the fact that all the agents are solving the same bandit problem, while CB with instant sharing has a large communication-cost dependency on the size of the network. DCB succesfully achieves near-optimal regret performance, while simultaneously reducing communication complexity by an order of magnitude in the size of the network. Moreover, DCCB generalises this regret performance at not extra cost in the order of the communication complexity.}
\end{figure}

\subsection{Proofs of Intermediary Results for DCB}
\label{ap:intermediary}
\begin{proof}[Proof of Proposition \ref{prop:concentration}]
This follows the proof of Theorem 2 in \cite{abbasi2011improved}, substituting appropriately weighted quantities.

For ease of presentation, we define the shorthand
$$\tX := (\sqrt{w_1}y_1,\dots,\sqrt{w_n}y_n)
\text{ and } \teta = (\sqrt{w_1}\eta_1,\dots,\sqrt{w_n}\eta_n)\tr,$$
where the $y_i$ are vectors with norm less than $1$, the $\eta_i$ are $R$-subgaussian, zero mean, random variables, and the $w_i$ are positive real numbers. Then, given samples $(\sqrt{w_1}y_1, \sqrt{w_1}(\theta y_1+\eta_1)), \dots, (\sqrt{w_n}y_n, \sqrt{w_n}(\theta y_n+\eta_n))$, the maximum likelihood estimate of $\theta$ is
\begin{align*}
\ttheta :&= (\tX\tX\tr+I)^{-1}\tX(\tX\tr\theta +\teta)\\
 &=  (\tX\tX\tr + I)^{-1}\tX\teta + (\tX\tX\tr+I)^{-1} (\tX\tX\tr+I)\theta -  (\tX\tX\tr+I)^{-1}\theta\\
 &=  (\tX\tX\tr+I)^{-1}\tX\teta + \theta -  (\tX\tX\tr+I)^{-1}\theta
\end{align*}
So by Cauchy-Schwarz, we have, for any vector $x$,
\begin{align}\label{eq:concentration_1}
x\tr(\ttheta - \theta) &= \langle x, \tX \teta \rangle_{ (\tX\tX\tr+I)^{-1}} - \langle x ,\theta \rangle_{ (\tX\tX\tr+I)^{-1}}\\
 &\le \| x \|_{(\tX\tX\tr+I)^{-1}} \left(\|\tX \teta \|_{(\tX\tX\tr+I)^{-1}} + \| \theta \|_{(\tX\tX\tr+I)^{-1}}\right)
\end{align}
Now from Theorem 1 of \cite{abbasi2011improved}, we know that with probability $1-\delta$
\begin{align*}
\|\tX \teta \|_{(\tX\tX\tr+I)^{-1}}^2  \le W^2R^2 2\log\sqrt{\frac{\det(\tX\tX\tr+I)}{\delta^2}}.
\end{align*}
where $W = \max_{i = 1, \dots, n}{w_i}$.
So, setting $x = (\tX\tX\tr+I)^{-1}(\ttheta - \theta) $, we obtain that with probability $1-\delta$
\begin{align*}
\| \ttheta - \theta \|_{(\tX\tX\tr+I)^{-1}}
	\le W R \left(2\log\sqrt{\frac{\det(\tX\tX\tr+I)}{\delta^2}}\right)^{\frac{1}{2}}
		+\|\theta\|_2
\end{align*}
since~\footnote{$\lambda_{\min}(\;\cdot\;)$ denotes the smallest eigenvalue of its argument.}
\begin{align*}
\| x \|_{(\tX\tX\tr+I)^{-1}}  \| \theta \|_{(\tX\tX\tr+I)^{-1}}
	&\leq \| x \|_2 \lambda_{\min}^{-1}(\tX\tX\tr+I) \| \theta \|_2 \lambda_{\min}^{-1} (\tX\tX\tr+I) \\
	&\leq \| x \|_2 \| \theta \|_2.
\end{align*}
Conditioned on the values of the weights, the statement of Proposition \ref{prop:concentration} now follows by substituting appropriate quantities above, and taking the probability over the distribution of the subGaussian random rewards. However, since this statement holds uniformly for any values of the weights, it holds also when the probability is taken over the distribution of the weights.
\end{proof}

\begin{proof}[Proof of Lemma \ref{lem:cov-matrix-det}]
Recall that $\tA_t^i$ is constructed from the contexts chosen from the first $\tau(t)$ rounds, across all the agents. Let $i'$ and $t'$ be arbitrary indices in $V$ and $\{1,\dots,\tau(t)\}$, respectively.
\begin{enumerate}[(i)]
	\item We have
\begin{align*}
\det \left(\tA_t^i \right)
	=& \det \left(\tA_t^i - \left(w_{i,t}^{i',t'} -1\right)x_{t'}^{i'} \left(x_{t'}^{i'}\right)\tr
														+ \left(w_{i,t}^{i',t'} -1\right)x_{t'}^{i'} \left(x_{t'}^{i'}\right)\tr  \right)\\
	=&  \det \left(\tA_t^i - \left(w_{i,t}^{i',t'} -1\right)x_{t'}^{i'} \left(x_{t'}^{i'}\right)\tr\right)\\
	  &\qquad.\left(1 + \left(w_{i,t}^{i',t'} -1\right)
	  								\|x_{t'}^{i'}\|_{\left(\tA_t^i - \left(w_{i,t}^{i',t'} -1\right)x_{t'}^{i'} \left(x_{t'}^{i'}\right)\tr\right)^{-1}}
	  						\right)
\end{align*}
The second equality follows using the identity $\det(I+ cB^{1/2}x x\tr B^{1/2}) = (1+c\|x\|_{B})$, for any matrix $B$, vector $x$, and scalar $c$. 
Now, we repeat this process for all $i'\in V$ and $t'\in\{1,\dots,\tau(t)\}$ as follows.
Let $(t_1,i_1), \dots, (t_{|V|\tau(t)},i_{|V|\tau(t)})$ be an arbitrary enumeration of $V \times \{1, \dots, \tau(t)\}$, let $B_0 = \tA_t^i$, and $B_{s} = B_{s-1} - (w_{i,t}^{i_s,t_s} - 1)x_{t_s}^{i_s}\left(x_{t_s}^{i_s}\right)\tr$ for $s= 1, \dots, |V|\tau(t)$.
Then $B_{|V|\tau(t)} = A_{\tau(t)}$, and by the calculation above we have
\begin{align*}
  \det \left(\tA_t^i \right)
=&\det \left(A_{\tau(t)}\right)
  \prod_{s=1}^{|V|\tau(t)}
    \left(
    1 +
    \left(w_{i,t}^{i_s,t_s} -1\right)
    \|x_{t_s}^{i_s}\|_{\left(B_s\right)^{-1}}
  \right)
\\
\le&  
  \det \left(A_{\tau(t)}\right)
  \exp
  \left(
    \sum_{s=1}^{|V|\tau(t)}
      \left(w_{i,t}^{i_s,t_s} -1\right)
      \|x_{t_s}^{i_s}\|_{\left(B_s\right)^{-1}}
  \right)
\\
	\le&\exp\left(
			     \sum_{t'=1}^{\tau(t)}
			     		\sum_{i' = 1}^{|V|}
								\left|w_{i,t}^{i',t'} -1\right|
				  \right)
						\det \left(A_{\tau(t)}\right)
\end{align*}
%
%
%
	\item Note that for vectors $x,y$ and a matrix $B$, by the Sherman-Morrison Lemma, and Cauchy-Schwarz inequality we have that:
\begin{align}
x\tr (B+y y\tr )^{-1}x
	= x\tr B^{-1}x - \frac{x\tr B^{-1}yy\tr B^{-1}x}{1+y\tr B^{-1}y}
	& \ge x\tr B^{-1}x - \frac{x\tr B^{-1}xy\tr B^{-1}y}{1+y\tr B^{-1}y}\nonumber\\
	& = x\tr B^{-1}x(1+y\tr B^{-1}y)^{-1}\label{eq:sher-morr-trick}
\end{align}
Taking
$$B = \left(\tA_t^i - \left(w_{i,t}^{i',t'} -1\right)x_{t'}^{i'} \left(x_{t'}^{i'}\right)\tr\right)
	\text{ and } y = \sqrt{w_{i,t}^{i',t'} -1}x_{t'}^{i'},$$
and using that $y\tr B^{-1} y \le \lambda_{min}(B)^{-1}y\tr y$,
by construction, we have that, for any $t'\in\{1,\dots,\tau(t)\}$ and $i'\in V$,
\begin{align*}
x\tr \left( \tA_t^i  \right)^{-1} x
	\ge x\tr \left(\tA_t^i - \left(w_{i,t}^{i',t'} -1\right)x_{t'}^{i'} \left(x_{t'}^{i'}\right)\tr\right)^{-1}
					x(1+ |w_{i,t}^{i',t'} -1| )^{-1}.
\end{align*}
Performing this for each $i'\in V$ and $t'\in\{1,\dots,\tau(t)\}$, taking the exponential of the logarithm  and using that $\log (1+a) \leq a$ like in the first part finishes the proof.
\end{enumerate}
\end{proof}


\subsection{Proof of Theorem \ref{thm:DCCB_main}}
\label{ap:DCCB_proof}

Throughout the proof let $i$ denote the index of some arbitrary but fixed agent, and $k$ the index of its cluster.

\paragraph{Step 1: Show the true clustering is obtained in finite time. }
First we prove that with probability $1-\delta$, the number of times agents in different clusters share information is bounded.
Consider the statements
\begin{align}\label{eq:lem_sharing_bound_1}
\forall i,i'\in V,\ \forall t,\ 
\left(\| \hattheta_{{local},t}^i - \hattheta_{local,t}^{i'} \|
			> c_{\lambda}^{thresh}(t)
		\right)
		\implies i'\notin U^k
\end{align}
and,
\begin{align}\label{eq:lem_sharing_bound_2}
\forall t\ge C(\gamma,\lambda,\delta) = {c_{\lambda}^{thresh}}^{-1}\left(\frac{\gamma}{2}\right),\ i'\notin U^k,\ 
\| \hattheta_{{local},t}^i - \hattheta_{local,t}^{i'} \|
			> c_{\lambda}^{thresh}(t).
\end{align}
where $c_{\lambda}^{thresh}$ and $A_\lambda$ are as defined in the main paper.
Lemma 4 from \cite{gentile2014online} proves that these two statements hold under the assumptions of the theorem with probability $1-\delta/2$.

Let $i$ be an agent in cluster $U^k$. Suppose that \eqref{eq:lem_sharing_bound_1} and \eqref{eq:lem_sharing_bound_2} hold. Then we know that at time $t = \lceil C(\gamma,\lambda,\delta) \rceil$, $U^k\subset V_t^i$. Moreover, since the sharing protocol chooses an agent uniformly at random from $V_t^i$ independently from the history before time $t$, it follows that the time until $V_t^i = U^k$ can be upper bounded by a constant $C = C(|V|,\delta)$ with probability $1-\delta/2$. So it follows that there exists a constant $C = C(|V|,\gamma,\lambda,\delta)$ such that the event
$$E := \{ \text{\eqref{eq:lem_sharing_bound_1} and \eqref{eq:lem_sharing_bound_2} hold, and }
					( t\ge C(|V|,\gamma,\lambda,\delta) \implies V_t^i = U^k )
					\}$$
holds with probability $1-\delta$.

\textbf{Step 2: Consider the properties of the weights after clustering. }
On the event $E$, we know that each cluster will be performing the algorithm DCB within its own cluster for all $ t>C(\gamma,|V|)$. Therefore, we would like to directly apply the analysis from the proof of Theorem \ref{thm:Main} from this point. In order to do this we need to show that the weights, $w_{i,t}^{i',t'}$, have the same properties after time $C = C(\gamma,|V|,\lambda,\delta)$ that are required for the proof of Theorem \ref{thm:Main}.

\begin{lemma}\label{lem:DCCB_prpts}
Suppose that agent $i$ is in cluster $U^k$. Then, on the event $E$,
\begin{enumerate}[(i)]
	\item for all $t>C(|V|,\gamma,\lambda,\delta)$ and $i'\in V\setminus U^k$, $w_{i,t}^{i',t'} = 0$;
	\item for all $t'\ge C(|V|,\gamma,\lambda,\delta)$ and $i'\in U^k$, $\sum_{i\in U^k} w_{i,C(|V|,\gamma)}^{i',t'} = |U^k|$;
	\item for all $t\ge t'\ge C(|V|,\gamma,\lambda,\delta)$ and $i'\in U^k$, the weights $w_{i,t}^{i',t'}$, $i\in U^k$, are i.d.. 
\end{enumerate}
\end{lemma}
\begin{proof}
See Appendix \ref{ap:DCCB_intermediary}.
\end{proof}

We must deal also with what happens to the information gathered before the cluster has completely discovered itself. To this end, note that we can write, supposing that $\tau(t)\ge C(|V|,\gamma,\lambda,\delta)$,
\begin{align}\label{eq:DCCB_A-tilde-form}
\tA_t^i
	:=\sum_{i' \in U^k} \frac{w_{i,t}^{i',C}}{|U^k|} \tA_{C}^{i'} 
		+ \sum_{t'=C + 1}^{\tau(t)} \sum_{i' \in U^k}  w_{i,t}^{i',t'} x_{t'}^{i'}\left( x_{t'}^{i'}\right)\tr.
\end{align}
Armed with this observation we show that the fact that sharing within the appropriate cluster only begins properly after time $C =C(|V|,\gamma,\lambda,\delta)$ the influence of the bias is unchanged:
\begin{lemma}[Bound on the influence of general weights]\label{lem:DCCB_cov-matrix-det}
On the event $E$, for all $i\in V$ and $t$ such that $T(t)\ge C(|V|,\gamma,\lambda,\delta)$,
\begin{enumerate}[(i)]
	\item $\det \left(\tA_t^i \right)
					\le \exp \left(\sum\limits_{t'=C}^{\tau(t)}
													 \sum\limits_{i' \in U^k} \left| w_{i,t}^{i',t'} - 1 \right|
										  \right) \det\left( A^k_{\tau(t)} \right)$,
	\item and $\| x_t^i \|_{\left(\tA_t^i\right)^{-1}}^2
							\le \exp \left(\sum\limits_{t'=C}^{\tau(t)}
													  \sum\limits_{i' \in U^k}  \left| w_{i,t}^{i',t'} - 1 \right|
										 		  \right) \| x_t^i \|_{\left(A^k_{\tau(t)}\right)^{-1}}^2$. 
\end{enumerate}
\end{lemma}
\begin{proof}
See Appendix \ref{ap:DCCB_intermediary}.
\end{proof}

The final property of the weights required to prove Theorem \ref{thm:Main} is that their variance is diminishing geometrically with each iteration. For the analysis of DCB this is provided by Lemma 4 of \cite{szorenyi2013gossip}, and, using Lemma \ref{lem:DCCB_prpts}, we can prove the same result for the weights after time $C = C(|V|,\gamma,\lambda,\delta)$:

\begin{lemma}\label{lem:DCCB_lemma_4}
Suppose that agent $i$ is in cluster $U^k$. Then, on the event $E$, for all $t\ge C = C(|V|,\gamma,\lambda,\delta)$ and $t'<t$, we have
\begin{align*}
\bE\left( (w_{i,t}^{j,t'} - 1)^2\right) \le \frac{|U^k|}{2^{t - \max\{t',C\}}}.
\end{align*}
\end{lemma}
\begin{proof}
Given the properties proved in Lemma \ref{lem:DCCB_prpts}, the proof is identical to the proof of Lemma 4 of \cite{szorenyi2013gossip}.
\end{proof}

\textbf{Step 3: Apply the results from the analysis of DCB. }
We can now apply the same argument as in Theorem \ref{thm:Main} to bound the regret after time $C = C(\gamma,|V|,\lambda,\delta)$. The regret before this time we simply upper bound by $|U^k|C(|V|,\gamma,\lambda,\delta)\|\theta\|$. We include the modified sections bellow as needed.

Using Lemma \ref{lem:DCCB_lemma_4}, we can control the random exponential constant in Lemma \ref{lem:DCCB_cov-matrix-det}, and the upper bound $W(T)$:
\begin{lemma}[Bound in the influence of weights under our sharing protocol]\label{lem:DCCB_weight-norm-sum-bound}
Assume that $t\ge C(\gamma,|V|,\lambda\delta)$. Then on the event $E$, for some constants $0<\delta_{t'}<1$, with probability $1-\sum_{t' = 1}^{\tau(t)} \delta_{t'}$
\begin{gather*}
\sum_{t'=C}^{\tau(t)} \sum_{i' \in U^k} \left| w_{i,t}^{i',t'} - 1 \right|
	\le |U^k|^{\frac{3}{2}} \sum_{t'=C}^{\tau(t)} \sqrt{\frac{2^{-(t-\max\{t',C\})}}{\delta_{t'}}},\\
\text{ and }
W({\tau(t)})
	\le 1+ \max_{C\le t' \le {\tau(t)}} \left\{ |U^k|^{\frac{3}{2}}\sqrt{\frac{2^{-(t-\max\{t',C\})}}{\delta_{t'}}}\right\}.
\end{gather*}
\end{lemma}

In particular, for any $1>\delta>0$, choosing $\delta_{t'} =  \delta 2^{-(t-\max\{t',C\})/2}$, and $\tau(t) =  t - c_1\log_2 c_2t$ we conclude that with probability $1-(c_2t)^{-c_1/2}\delta/(1-2^{-1/2})$, for any $t>C + c_1\log_2 (c_2 C) $,
\begin{gather}
\sum_{i' \in U^k} \sum_{t'=C}^{\tau(t)} \left| w_{i,t}^{i',t'} - 1 \right|
	\le  \frac{|U^k|^{\frac{3}{2}}(c_2t)^{- \frac{c_1}{4}}}
					{(1 - 2^{-\frac{1}{4}})\sqrt{\delta}},
\text{ and }
W({\tau(t)})
	\le 1+ \frac{|U^k|^{\frac{3}{2}}(c_2t)^{- \frac{c_1}{4}}}
						{\sqrt{\delta}}.\label{eq:DCCB_A_bound}
\end{gather}
Thus lemmas \ref{lem:DCCB_cov-matrix-det} and \ref{lem:DCCB_weight-norm-sum-bound} give us control over the bias introduced by the imperfect information sharing.
Applying lemmas \ref{lem:DCCB_cov-matrix-det} and \ref{lem:DCCB_weight-norm-sum-bound}, we find that with probability $1-(c_2t)^{-c_1/2}\delta/(1-2^{-1/2})$:
\begin{align}\label{eq:DCCB_sr_high_prob}
&\rho_t^i
	\le 2\exp\left(\frac{|U^k|^{\frac{3}{2}}}
									  {(1 - 2^{-\frac{1}{4}})c_2^{ \frac{c_1}{4}}t^{\frac{c_1}{4}}\sqrt{\delta}}
							\right) 
			\| x_t^i \|_{\left(A_{\tau(t)}^i \right)^{-1}} \\
  		 &\quad 
  		 .\left[
  		 			\left(1+ \frac{|U^k|^{\frac{3}{2}}}{(1 - 2^{-\frac{1}{4}})c_2^{ \frac{c_1}{4}}t^{\frac{c_1}{4}}\sqrt{\delta}}\right)
  					\left[R\sqrt{2\log\left(\exp\left(\frac{|U^k|^{\frac{3}{2}} }
  																		{(1 - 2^{-\frac{1}{4}})c_2^{ \frac{c_1}{4}}t^{\frac{c_1}{4}}\sqrt{\delta}}
  																		\right)
  														\frac{\det\left(A_{\tau(t)} \right)^{\frac{1}{2}}}
  								    							{\delta}
  								    					\right)
  								    	}
  							+ \|\theta\|
  							\right]
  		  		\right].\nonumber
\end{align}

\textbf{Step 4: Choose constants and sum the simple regret.}
Choosing again $c_1 = 4$, $c_2 = |V|^{\frac{3}{2}}$, and setting
$
N_{\delta}=\frac{1}{(1 - 2^{-\frac{1}{4}})\sqrt{\delta}},
$
we have on the event $E$, for all $t\ge \max\{N_{\delta},  C + 4\log_2 (|V|^{\frac{3}{2}} C)\}$, with probability $1-(|V|t)^{-2}\delta/(1-2^{-1/2})$
\begin{align*}
\rho_t^i \le  4e  \| x_t^i \|_{ \left(A^k_{t-1} + \sum_{i'=1}^{i-1}x_{t}^{i'} \left(x_{t}^{i'}\right)\tr \right)^{-1}}
  							\left(\beta(t) +R\sqrt{2}\right),
\end{align*}
where $\beta(\cdot)$ is as defined in the theorem statement.
Now applying Cauchy-Schwarz, and Lemma 11 from \cite{abbasi2011improved} yields that on the event $E$, with probability $1 - \left(1 + \sum_{t = 1}^\infty (|V|t)^{-2}/(1-2^{-1/2})\right) \delta\ge 1 - 3\delta$,
\begin{align*}
\cR_t \le& \left(\max\{N_{\delta},  C + 4\log_2 (|V|^{\frac{3}{2}} C)\}
						  + 2\left(4|V|d\log\left(|V|t\right)\right)^3
						  \right)\|\theta\|_2\\
 					&+ 4e \left(\beta(t) + R\sqrt{2} \right)\sqrt{|U^k|t \left(2\log\left( \det\left(A^k_t \right)\right)\right) }.
\end{align*}
Replacing $\delta$ with $\delta/6$, and combining this result with Step 1 finishes the proof.

\subsection{Proofs of Intermediary Results for DCCB}
\label{ap:DCCB_intermediary}
\begin{proof}[Proof of Lemma \ref{lem:DCCB_prpts}]
Recall that whenever the pruning procedure cuts an edge, both agents reset their buffers to their local information, scaled by the size of their current neighbour sets. (It does not make a difference practically whether or not they scale their buffers, as this effect is washed out in the computation of the confidence bounds and the local estimates. However, it is convenient to assume that they do so for the analysis.) Furthermore, according to the pruning procedure, no agent will share information with another agent that does not have the same local neighbour set.

On the event $E$, there is a time for each agent, $i$, before time $C = C(\gamma,|V|,\lambda\delta)$ when the agent resets its information to their local information, and their local neighbour set becomes their local cluster, i.e. $V_t^i = U^k$. After this time, this agent will only share information with other agents that have also set their local neighbour set to their local cluster. This proves the statement of part (i).

Furthermore, since on event $E$, after agent $i$ has identified its local neighbour set, i.e. when $V_t^i = U^k$, the agent only shares with members of $U^k$, the statements of parts (ii) and (iii) hold by construction of the sharing protocol.
\end{proof}

\begin{proof}[Proof of Lemma \ref{lem:DCCB_cov-matrix-det}]
The result follows the proof of Lemma \ref{lem:cov-matrix-det}. For the the iterations until time $C = C(\gamma,|V|,\lambda\delta)$ is reached, we apply the argument there. For the final step we require two further inequalities.

First, to finish the proof of part (i) we note that,
\begin{align*}
\det&\left( (A^k_T - A^k_{C})
	+  \sum_{i' \in U^k} \frac{w_{i,t}^{i',C(\gamma,|V|)}}{|U^k|}
												\tA_{C}^{i'} \right)
	= \det\left( A^k_T
					+  \sum_{i' \in U^k} \frac{w_{i,t}^{i',C}-1}{|U^k|}
																\tA_{C}^{i'} \right)\\
	&=  \det\left( A^k_T \right)
			\det\left( I
					+  \sum_{i' \in U^k} \frac{w_{i,t}^{i',C}-1}{|U^k|}
																{A^k_T}^{-\frac{1}{2}}\tA_{C}^{i'}{A^k_T}^{-\frac{1}{2}} \right)\\
	&\le  \det\left( A^k_T \right)
			\det\left( I
					+  \left[\sum_{i' \in U^k} \left|w_{i,t}^{i',C}-1\right|\right]
																{A^k_T}^{-\frac{1}{2}}\sum_{i' \in U^k}\frac{\tA_{C}^{i'}}{|U^k|}
																																		{A^k_T}^{-\frac{1}{2}} \right)
	\le  \det\left( A^k_T \right)
			 \left( 1 +  \sum_{i' \in U^k} \left|w_{i,t}^{i',C}-1\right| \right)	.
\end{align*}
For the first equality we have used that $|U^k|A^k_{C} = \sum_{i' \in U^k}\tA_{C}^{i'}$; for the first inequality we have used a property of positive definite matrices; for the second inequality we have used that $1$ upper bounds the eigenvalues of ${A^k_T}^{-1/2}A^k_{C}{A^k_T}^{-1/2} $.

Second, to finish the proof of part (ii), we note that, for any vector $x$,
\begin{align*}
x\tr&\left( A^k_{\tau(t)}
					+  \sum_{i' \in U^k} \frac{w_{i,t}^{i',C}-1}{|U^k|}
																\tA_{C}^{i'}
				\right)^{-1} x\\
	&= \left({A^k_{\tau(t)}}^{-\frac{1}{2}}x\right)\tr
			\left( I
					+  \sum_{i' \in U^k} \frac{w_{i,t}^{i',C}-1}{|U^k|}
																{A^k_{\tau(t)}}^{-\frac{1}{2}}\tA_{C}^{i'}{A^k_{\tau(t)}}^{-\frac{1}{2}}
					\right)^{-1}
			\left({A^k_{\tau(t)}}^{-\frac{1}{2}}x\right)\\
	&\ge \left({A^k_{\tau(t)}}^{-\frac{1}{2}}x\right)\tr
			\left( I
					+ \sum_{i' \in U^k} \frac{\left|w_{i,t}^{i',C}-1\right|}{|U^k|}
																{A^k_{\tau(t)}}^{-\frac{1}{2}}\tA_{C}^{i'}{A^k_{\tau(t)}}^{-\frac{1}{2}}
					\right)^{-1}
			\left({A^k_{\tau(t)}}^{-\frac{1}{2}}x\right)\\
	&\ge \left( 1 +  \sum_{i' \in U^k} \left|w_{i,t}^{i',C}-1\right| \right)^{-1}
			  	x\tr {A^k_{\tau(t)}}^{-1} x.
\end{align*}
The first inequality here follows from a property of positive definite matrices, and the other steps follow similarly to those in the inequality that finished part (i) of the proof.
\end{proof}

\end{document}